\documentclass[twoside]{article} 

\usepackage[accepted]{aistats2017}
\usepackage{lmodern}
\usepackage[T1]{fontenc}
\usepackage[latin9]{inputenc}

\usepackage{mathtools}
\mathtoolsset{showonlyrefs}

\usepackage{hyperref}
\usepackage{url}

\usepackage{dsfont} 
\usepackage{xspace, amsmath, amssymb, amsthm, bbm, bm}
\usepackage{color}
\usepackage{fancyvrb}

\usepackage{pdfsync}
\pdfoutput=1

\usepackage{graphicx}
\usepackage{tabularx}

\usepackage{algorithm}
\usepackage{algorithmic}

\usepackage{multirow}
\usepackage{hhline}

\usepackage{booktabs}

\usepackage{capt-of}

\newtheorem{lem}{Lemma}
\newtheorem{thm}{Theorem}
\newtheorem{cor}{Corollary}

\allowdisplaybreaks 

\def\one{\ensuremath{\mathds{1}\xspace}} 
 
\def\larrow{\ensuremath{\leftarrow}\xspace}

\def\T{\ensuremath{\top}}  

\newtheorem{assump}{Assumption}

\def\bfpi{{\ensuremath{\bm{\pi}}\xspace}}

\def\dt{{\ensuremath{\delta}\xspace} }
\def\sm{{\ensuremath{\setminus}\xspace} }

\usepackage[export]{adjustbox}

\def\lfl{\lfloor} 
\def\rfl{\rfloor}

\usepackage{pbox} 

\newcommand{\fr}[2]{ { \frac{#1}{#2} }}
\def\lt{\left}
\def\rt{\right}
\def\gam{{\ensuremath{\gamma}\xspace} }

\makeatletter
\newcommand{\vast}{\bBigg@{3}}
\newcommand{\Vast}{\bBigg@{4}}
\makeatother

\def\dsR{{{\mathds{R}}}}

\def\w{{{\mathbf w}}}
\def\x{{{\mathbf x}}}

\def\z{{{\mathbf z}}}
\def\la{{\langle}}
\def\ra{{\rangle}}

\def\calL{\ensuremath{\mathcal{L}}\xspace}

\usepackage{pifont}
\newcommand{\cmark}{\ding{51}}%
\newcommand{\gyes}{{\color[rgb]{0,.8,0}\cmark}}
\newcommand{\xmark}{\ding{55}}%
\newcommand{\rno}{{\color[rgb]{.8,0,0}\xmark}}

\def\bfell{{{\boldsymbol\ell}}}
\def\p{{{\mathbf p}}}
 
\def\cD{\ensuremath{\mathcal{D}}\xspace} 
\def\cW{\ensuremath{\mathcal{W}}\xspace} 
 
\def\e{{{\mathbf e}}}
\def\u{{{\mathbf u}}}

\def\w{{{\mathbf w}}}

\def\KL{\ensuremath{\mbox{KL}}\xspace}  
\def\cB{\ensuremath{\mathcal{B}}\xspace} 
\def\cI{\ensuremath{\mathcal{I}}\xspace} 
\def\cA{\ensuremath{\mathcal{A}}\xspace} 
\def\cJ{\ensuremath{\mathcal{J}}\xspace} 
\def\cM{\ensuremath{\mathcal{M}}\xspace}

\usepackage{enumitem}

\def\l{{{\boldsymbol \ell}}}

\def\hatp{\ensuremath{\widehat p}}
\def\hatbfp{\ensuremath{\widehat{\mathbf{p}}}}
\def\tilg{\ensuremath{\widetilde{g}}}
\def\tilG{\ensuremath{\widetilde{G}}}

\def\bfcI{\ensuremath{\boldsymbol{\mathcal{I}}}} 
\def\Active{\ensuremath{\text{Active}}}

\def\CBCE{\ensuremath{\text{CBCE}}}
\def\Regret{\ensuremath{\text{Regret}}}
\def\Wealth{\ensuremath{\text{Wealth}}}
\def\CB{\ensuremath{\text{CB}}}

\def\M{\ensuremath{\mathbf{M}}} 
\def\dsN{{{\mathds{N}}}}

\begin{document}

%

%

\twocolumn[

\aistatstitle{Improved Strongly Adaptive Online Learning using Coin Betting}

\aistatsauthor{ Kwang-Sung Jun \And Francesco Orabona \And Stephen Wright \And Rebecca Willett}

\aistatsaddress{ UW-Madison \And Stony Brook University \And UW-Madison \And UW-Madison } ]

\vspace{-10pt}
\begin{abstract}
  \vspace{-10pt}
  This paper describes a new parameter-free online learning algorithm for changing environments.
  In comparing against algorithms with the same time complexity as ours, we obtain a strongly adaptive regret bound that is a factor of at least $\sqrt{\log(T)}$ better, where $T$ is the time horizon.
  Empirical results show that our algorithm outperforms state-of-the-art methods in learning with expert advice and metric learning scenarios.
\end{abstract}
\vspace{-15pt}

\section{Introduction}
\vspace{-5pt}

Machine learning has made heavy use of the  i.i.d. assumption, but this assumption does not hold in many applications.
For example, in online portfolio management, stock price trends can vary unexpectedly, and the ability to track  changing trends and adapt to them are crucial in maximizing one's profit.
Another example is seen in  product reviews, where words describing product quality may change over time as products and customer's taste evolve.
Keeping track of the changes in the metric describing the relationship between review text and rating is crucial for improving analysis and quality of recommendations.

We consider the problem of adapting to a changing environment in the online learning context.
Let $\cD$ be the decision space, $\calL$ be loss functions that map $\cD$ to $\dsR$, and
$T$ be the target time horizon.
Let $\cA$ be an online learning algorithm and $\cW \subseteq \cD$ be the set of comparator decisions. (Often, $\cW = \cD$.)
We define the online learning problem in Figure~\ref{fig:ol}.
The usual goal of online learning is to find a strategy that compares favorably with the best fixed comparator in $\cW$, in hindsight.
Specifically, we seek a low value of the following (cumulative) static regret objective: 
$
\mbox{Regret}^\cA_T := \sum_{t=1}^T f_t(\x^\cA_t) - \min_{\w \in \cW} \sum_{t=1}^T f_t(\w).
$
\begin{figure}
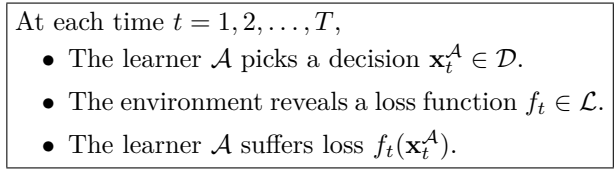

\fbox{\begin{minipage}{223pt}
At each time $t = 1,2,\ldots,T$,
\begin{itemize}[topsep=2pt,itemsep=0ex,partopsep=1ex,parsep=1ex]
  \item The learner $\cA$ picks a decision $\x^{\cA}_t \in \cD$.
  \item The environment reveals a loss function $f_t \in \calL$.
  \item The learner $\cA$ suffers loss $f_t(\x^{\cA}_t)$.
\end{itemize}
\end{minipage}}
\vspace{-4pt}
\caption{Online learning protocol}
\label{fig:ol}
\vspace{-36pt}
\end{figure}
When the environment is changing,  static regret is not a suitable measure, since it compares the learning strategy against a decision that is fixed. We need to make use of stronger notions of regret that allow for comparators to change over time. To define such notions, we introduce the notation
$[T] := \{1, \ldots,T\}$ and $[A..B] = \{A, A+1, \ldots, B\}$.
Daniely et al.~\cite{daniely15strongly} defined {\em strongly adaptive regret (SA-Regret)},  which requires an algorithm to have low (cumulative) regret over any contiguous time interval $I = [I_1..I_2] \subseteq [T]$.\footnote{
  Strongly adaptive regret is similar to the notion of adaptive regret introduced by~\cite{hazan07adaptive}, but emphasizes the dependency on the interval length $|I|$.
}
Another notion, $m$-shift regret~\cite{herbster98tracking}, measures instead the regret w.r.t. a comparator that changes at most $m$ times in $T$ time steps.
Note that the SA-Regret is a stronger notion than the $m$-shift regret since the latter can be derived directly from the former~\cite{luo15achieving,daniely15strongly}, as we show in our supplementary material.
We  define SA-Regret and $m$-shift regret precisely in Section~\ref{sec:prelim}.

Several generic online algorithms that adapt to a changing environment have been proposed recently.
Rather than being designed for a specific learning problem, these are ``meta algorithms'' that take \emph{any} online learning algorithm as a black-box and turn it into an adaptive one.
We summarize the SA-Regret of existing meta algorithms in Table~\ref{tab:regret-meta-sa}.
In particular, the pioneering work of Hazan et al.~\cite{hazan07adaptive} introduced the \emph{adaptive regret}, that is a slightly weaker notion than the SA-Regret, and proposed two meta algorithms called FLH and AFLH. 
However, their SA-Regret depends on $T$ rather than $|I|$.
The SAOL  approach of~\cite{daniely15strongly} improves the SA-Regret to $O\lt(\sqrt{(I_2-I_1)\log^2(I_2)}\rt)$.

\begin{table*}
     {\centering\footnotesize
\begin{tabular}{|c|c|c|c|c|c|c|} \hline
Algorithm   & $m$-shift regret & Time               & Agnostic to $m$ \\ \hline
Fixed Share~\cite{herbster98tracking,cesa-bianchi12mirror} & $\sqrt{mT(\log N + \log T)}$ & $NT$     & \rno           \\
           & $\sqrt{m^2 T(\log N + \log T)}$ & $NT$ & \gyes           \\ \hline
GeneralTracking$\la$EXP$\ra$~\cite{gyorgy12efficient} & $\sqrt{mT(\log N + m\log^2 T)}$ & $NT\log T$&\gyes\\
  & $\sqrt{mT(\log N + \log^2 T)}$ & $NT\log T$ & \rno\\ 
($\gam\in(0,1)$) & $\sqrt{\fr{1}{\gam}mT(\log N + m\log T)}$ & $NT^{1+\gam}\log T$ & \gyes\\ 
                 & $\sqrt{\fr{1}{\gam}mT(\log N + \log T)}$ & $NT^{1+\gam}\log T$ &\rno\\ \hline
ATV~\cite{luo15achieving} & $\sqrt{mT(\log N + \log T)}$ & $NT^2$ & \gyes       \\ \hline
SAOL$\la$MW$\ra~\cite{daniely15strongly}$ & $\sqrt{mT(\log N + \log^2 T)}$ & $NT\log T$        & \gyes       \\ \hline
\CBCE$\la\CB\ra$ (ours)      & $\sqrt{mT(\log N + \log T)}$ & $NT\log T$   & \gyes       \\ \hline
\end{tabular}
\vspace{-5pt}
\caption{$m$-shift regret bounds of LEA algorithms. 
  Our proposed algorithm achieves the best regret among those with the same time complexity and does not need to know $m$.
  Each quantity omits constant factors.  
  Agnostic to $m$ means that an algorithm does not need to know the number $m$ of switches in the best expert. 
}
\label{tab:tracking}
}
\vspace{-12pt}
\end{table*} 
\begin{table}
  {\footnotesize\centering
\begin{tabular}{|c|c|c|c|} \hline
Algorithm & SA-Regret order                          & Time factor \\ \hline
FLH~\cite{hazan07adaptive}    & $ \sqrt{T \log T}$                 & $T$         \\  
AFLH~\cite{hazan07adaptive}   & $ \sqrt{T \log T } \log (I_2-I_1)$ & $\log T $   \\  \hline
SAOL~\cite{daniely15strongly} & $ \sqrt{(I_2 - I_1)\log^2 (I_2)}$  & $\log T $   \\  \hline
$\CBCE$ (ours)                         & $ \sqrt{(I_2-I_1)\log (I_2)}$      & $\log T $ \\  \hline
\end{tabular}
\vspace{-5pt}
\caption{SA-Regret bounds of meta algorithms on $I \subseteq [T]$.
  Our proposed algorithm achieves the best SA-Regret.
  We show the part of the regret due to the meta algorithm only, not the black-box.
  The last column is the multiplicative factor in the time complexity introduced by the meta algorithm.
}
\label{tab:regret-meta-sa}
}
\vspace{-8pt}
\end{table}

In this paper, we propose a new meta algorithm called {\em Coin Betting for Changing Environment} ($\CBCE$) that combines the sleeping bandits idea~\cite{blum97empirical,freund97using} with the Coin Betting (CB) algorithm~\cite{orabona16from}.
The SA-Regret of CBCE is better by a factor $\sqrt{\log(I_2)}$ than that of SAOL, as shown in Table~\ref{tab:regret-meta-sa}. 
We present our extension of CB to sleeping bandits and prove its regret bound in Section~\ref{sec:coinbetting}. This result leads to the improved SA-Regret bound of CBCE in Section~\ref{sec:cbce}.

Our improved bound yields a number of improvements in various online
learning problems. In describing these improvements, we denote by
$\cM\la\cB\ra$ a complete algorithm assembled from meta algorithm $\cM$ and black-box
$\cB$.

Consider the learning with expert advice (LEA) problem with $N$
experts.  CBCE with black-box CB ($\CBCE\la\CB\ra$, in our notation)
has the $m$-shift regret
\vspace{-5pt}
\begin{align*}
  O\sqrt{mT(\log N + \log T)} 
\end{align*}
and  time complexity $O(NT \log T)$.
This regret is a factor $\sqrt{\log T}$ better than  those algorithms with the same time complexity. 
Although AdaNormalHedge.TV (ATV) and Fixed Share achieve the same regret, the former has larger time complexity, and the latter requires prior knowledge of the number of shifts $m$.
We summarize the $m$-shift regret bounds of various algorithms in Table~\ref{tab:tracking}.

In Online Convex Optimization (OCO) with $G$-Lipschitz loss functions over a convex set $D \in \dsR^d$ of diameter $B$, online gradient descent has regret $O(BG\sqrt{T})$.
$\CBCE$ with black-box OGD  ($\CBCE\la\text{OGD}\ra$) then has the following SA-Regret:
\vspace{-5pt}
\begin{align*}
O((BG + \sqrt{\log(I_2)})\sqrt{|I|}) \;,
\end{align*}
which improves by a factor $\sqrt{\log(I_2)}$ over SAOL$\la$OGD$\ra$.

In Section~\ref{sec:expr}, we compare $\CBCE$ empirically to a number of meta algorithms within a changing environment in two online learning problems: $(i)$ LEA and $(ii)$ Mahalanobis metric learning.
We observe that $\CBCE$ outperforms the state-of-the-art methods in both tasks, thus confirming our theoretical findings.

\vspace{-10pt}
\subsection{Preliminaries}
\label{sec:prelim}
\vspace{-5pt}

In this section we define some concepts that will be used in the rest of the paper.

A learner's SA-Regret is obtained by evaluating static regret on all (contiguous) time intervals $I = [I_1..I_2] \subseteq [T]$ of a given length $\tau$.
Specifically, the SA-Regret of an algorithm $\cA$ at time $T$ for length $\tau$ is
\vspace{-3pt}
\begin{align}\label{sa-regret} 
&\mbox{SA-Regret}_T^\cA(\tau) \notag \\
&:= \max_{I \subseteq [T]: |I| = \tau} \lt( \sum_{t\in I} f_t(\x^\cA_t) - \min_{\w\in\cW} \sum_{t\in I} f_t(\w) \rt) \;.
\end{align}
We call an algorithm \emph{strongly adaptive} if it has a low value of SA-Regret.
We call $\w_{1:T} := \{\w_1, \ldots, \w_T\}$ an \emph{$m$-shift sequence} if it changes at most $m$ times, that is,  $\sum_{j=1}^{T-1} \one\{\w_j \neq \w_{j+1}\} \le m$.
We define
\vspace{-3pt}
\begin{align}
  &m\mbox{-Shift-Regret}^\cA_T \notag \\
  &:= \sum_{t=1}^T f_t(\x^\cA_t) - \min_{\w_{1:T} \in \cW^{T} \;:\; m\mbox{-shift seq.}} \sum_{t=1}^{T} f_t(\w_t) \;.
\end{align}

\section{A Meta Algorithm for Changing Environments}
\label{sec:meta}
\vspace{-5pt}

Let $\cB$ be a black-box online learning algorithm following the protocol in Figure~\ref{fig:ol}.
A trick commonly used in designing a meta algorithm $\cM$ for changing environments is to initiate a new instance of $\cB$ at every time step~\cite{hazan07adaptive,gyorgy12efficient,adamskiy12acloser}. 
That is, we  run $\cB$ independently for each interval $J$ in $\{[t..\infty] \mid t = 1,2,\ldots\}$.
Denote by $\cB_J$ the run of black-box $\cB$ on interval $J$.
A meta algorithm at time $t$ combines the decisions from the runs $\{\cB_J\}_{J \ni t}$ by weighted average.
The key idea is that at time $t$, some of the outputs $\cB \in \{\cB_J\}_{J\ni t}$ are not based on any data prior to time $t' < t$, so that if the environment changes at time $t'$,  those outputs may be given a larger weight by the meta algorithm, allowing it to  adapt more quickly to the change.
This trick requires updating of $t$ instances of the black-box algorithm at each time step $t$, leading to a factor-of-$t$ increase in the time complexity.
This factor can be reduced to $O(\log t)$ by restarting black-box algorithms on a carefully designed set of intervals such as the geometric covering intervals~\cite{daniely15strongly} (GC) or the data streaming technique~\cite{hazan07adaptive,gyorgy12efficient} (DS) that is a special case of a more general set of intervals considered in~\cite{veness13partition}.
While both GC and DS achieve the same goal as we show in our supplementary material,\footnote{Except for a subtle case, which we also discuss in our supplementary material.} we use the former as our starting point for ease of exposition.   

\vspace{-7pt}
\paragraph{Geometric Covering Intervals.}

Define $
  \cJ_k := \{[ \lt(i\cdot 2^k\rt) .. \lt((i+1)\cdot2^k - 1\rt) ]: i \in \mathds{N}\}$, $\forall k \in \{0,1,\ldots\}$ to be the collection of intervals of length $2^k$.
The geometric covering intervals~\cite{daniely15strongly} are 
\vspace{-5pt}
\begin{align*}
  \cJ := \bigcup_{k\in\{0,1,\ldots\}} \cJ_k \;.
\end{align*}
That is, $\cJ$ is the set of intervals of doubling length, with intervals of size $2^k$ exactly partitioning the set $\mathds{N} \sm \{1,\ldots,2^k-1\}$, see Figure~\ref{fig:gc}.

Define the set of intervals that includes time $t$ as $\mbox{Active}(t) := \{J\in\cJ: t \in J\} $.
One can easily show that $|\mbox{Active}(t)| = \lfl\log_2(t)\rfl + 1$.
Since at most $O(\log (t))$ intervals contain any given time point $t$, the time complexity of the meta algorithm is a factor $O(\log(t))$  larger than that of the black-box $\cB$.

The key result of the geometric covering intervals strategy is the following Lemma from~\cite{daniely15strongly}, which shows that  an arbitrary interval $I$ can be partitioned into a sequence of smaller blocks whose lengths successively double, then successively halve.
\vspace{.5em}
\begin{lem}\emph{(\cite[Lemma~5]{daniely15strongly})} \label{lem:geo-intv}
  Any interval $I \subseteq \mathds{N}$ can be partitioned into two finite sequences of disjoint and consecutive intervals, denoted $\{J^{(-a)}, J^{(-a+1)}, \ldots, J^{(0)}\}$ and $\{J^{(1)}, J^{(2)}, \ldots, J^{(b)}\}$ where $ \forall i \in [(-a)..b]$, we have   $J^{(i)} \in \cJ$ and $J^{(i)} \subset I$, such that
\vspace{-5pt}
  \begin{align*}
    |J^{(-i)}| / |J^{(-i+1)}| & \le 1/2, \quad i=1,2,\dotsc,a; \\
    |J^{(i+1)}| / |J^{(i)}| & \le 1/2, \quad i=1,2,\dotsc,b-1~.
  \end{align*}
\end{lem}
\begin{figure}[t]
{\fontsize{7.3}{9.3}
\begin{Verbatim}[commandchars=\\\{\},codes={\catcode`$=3\catcode`_=8}]
\;   1  2  3  4  5  6  7  8  9 10 11 12 13 14 15 16 17 18 ...
$\cJ_0$[ ][ ][ ][ ][ ][ ][ ][ ][ ][ ][ ][ ][ ][ ][ ][ ][ ][ ]...
$\cJ_1$   [    ][    ][    ][    ][    ][    ][    ][    ][  ... 
$\cJ_2$         [          ][          ][          ][        ...
$\cJ_3$                     [                      ][        ...
\end{Verbatim}
}
\vspace{-15pt}
\caption{Geometric covering intervals. Each interval is denoted by {\tt [ ]}.}
\label{fig:gc}
\vspace{-12pt}
\end{figure}
%
\paragraph{Regret Decomposition.}
We show now how to use the geometric covering intervals to decompose the SA-Regret of a complete algorithm $\cM\la\cB\ra$. We use the notation 
\vspace{-5pt}
$$R^\cA_I(\w) := \sum_{t\in I} f_t(\x^\cA_t) - \sum_{t\in I} f_t(\w)\;,$$
and from \eqref{sa-regret} we see that 
 $\mbox{SA-Regret}_T^\cA(\tau) = \max_{I \subseteq [T]: |I|=\tau } \max_{\w\in\cW} R^\cA_I (\w)$.

Denote by $\x^{\cB_J}_t$ the decision from black-box $\cB_J$ at time $t$ and by $\x^{\cM\la\cB\ra}_t$ the combined decision of the meta algorithm.
Since $\cM\la\cB\ra$ is a combination of a meta $\cM$ and a black-box $\cB$, its regret depends on both $\cM$ and $\cB$. 
Perhaps surprisingly, we can decompose the two sources of regret additively through the geometric covering $\cJ$, as we now describe.
Choose some  $I \subseteq [T]$, and let  $\bigcup_{i=-a}^{b} J^{(i)}$ be the partition of $I$
obtained from Lemma~\ref{lem:geo-intv}.
Then, the regret of $\cM\la\cB\ra$ on $I$ can be  decomposed as follows:
\vspace{-5pt}
\begin{align}
R^{\cM\la\cB\ra}_I(\w)  &= \sum_{t\in I} \lt(f_t(\x^{\cM\la\cB\ra}_t) - f_t(\w) \rt) \notag \\
&= \sum_{i=-a}^{b} \Bigg(\sum_{t\in J^{(i)}} f_t(\x^{\cM\la\cB\ra}_t) - f_t(\x^{\cB_{J^{(i)}}}_t) + \notag\\ 
&  \qquad f_t(\x^{\cB_{J^{(i)}}}_t) - f_t(\w) \Bigg) \notag\\
&= \underbrace{
      \sum_{i=-a}^{b} \underbrace{ \sum_{t\in J^{(i)}} \lt(f_t(\x^{\cM\la\cB\ra}_t) - f_t(\x^{\cB_{J^{(i)}}}_t) \rt) }_{=:\text{\small  (meta regret on $J^{(i)}$)}} 
   }_{=:\text{\small (meta regret on $I$)}}
      + \notag\\
  &\quad \sum_{i=-a}^{b} \underbrace{ \sum_{t\in J^{(i)}} \lt(f_t(\x^{\cB_{J^{(i)}}}_t) - f_t(\w)\rt) }_{=:\text{\small  (black-box regret on $J^{(i)}$)}} \;. \label{regret-decomposition}
\end{align}
(We purposely use symbol $J$ for intervals in $\cJ$ and  $I$ for a generic interval that is not necessarily in $\cJ$.)
The black-box regret on $J=[J_1..J_2]\in \cJ$ is exactly the standard regret for $T = |J|$, since the black-box run $\cB_J$ was started from time $J_1$.
Thus, in order to prove that a meta algorithm $\cM$ suffers low SA-Regret, we must show two things.
\begin{enumerate}[topsep=0pt,itemsep=0ex,partopsep=1ex,parsep=1ex]
  \item $\cM$ has low regret on interval $J\in\cJ$.
  \item The outer sum over $i$ in~\eqref{regret-decomposition} is  small for both the meta and the black-box.
\end{enumerate}
Daniely et al.~\cite{daniely15strongly} address the second issue above efficiently in their analysis.
They show that if the black-box regret on $J^{(i)}$ scales like $\tilde O(\sqrt{|J^{(i)}|})$ (where $\tilde O$ ignores logarithmic factors), then the second double summation of~\eqref{regret-decomposition} is\footnote{
  This is essentially the same argument as the ``doubling trick'' described in~\cite[Section~2.3]{cesa-bianchi06prediction}
} $\tilde O(\sqrt{|I|})$, which is perhaps the best one can hope for.
The same holds true for the meta algorithm.
Thus, it remains to focus on the first issue above, which is our main contribution. 

In the next two sections, we show how to design our meta algorithm. 
In Section~\ref{sec:coinbetting} we propose a novel method that incorporates sleeping bandits and the coin betting framework.
Section~\ref{sec:cbce}  describes how our method can be used as a meta algorithm with strongly adaptive regret guarantees.

\vspace{-5pt}
\section{Coin Betting Meets Sleeping Experts} 
\label{sec:coinbetting}
\vspace{-5pt}

Our  meta algorithm is an extension of the coin-betting framework~\cite{orabona16from} based on sleeping experts~\cite{blum97empirical,freund97using}. It is parameter-free (there is no explicit learning rate) and has near-optimal regret.
Our construction, described below, might also be of independent interest.

\vspace{-7pt}
\paragraph{Sleeping Experts.}
In the learning with expert advice (LEA) framework,
 the decision set is $\cD = \Delta^N$, an $N$-dimensional probability simplex of weights assigned to the experts.
To distinguish LEA from the general online learning problem, we use notation $\p_t$ in place of $\x_t$ and $h_t$ in place of $f_t$.
Let $\bfell_t := (\ell_{t,1}, \ldots, \ell_{t,N})^\T \in [0,1]^{N}$ be the vector of loss values of experts at time $t$ that is provided by the environment.
The learner's loss function is $h_t(\p) := \p^\T \bfell_t$.

Since $\p \in \cD$ is a probability vector, the learner's decision can be viewed as hedging between the $N$ alternatives. 
Let $\e_i$ be an indicator vector for dimension $i$; e.g., $\e_2 = (0,1,0,\ldots,0)^\T$.
In this notation, the comparator set $\cW$ is $\{\e_1, \ldots, \e_N\}$.
Thus, the learner aims to compete with a strategy that commits to a single expert for the entire time $[1..T]$.

The decision set may be nonconvex, for example, when $\cD = \{\e_1, \ldots, \e_N\}$ \cite[Section~3]{cesa-bianchi06prediction}.
In this case, no hedging is allowed; the learner must pick an expert.
To choose an element of this set, one could first choose an
element $p_t$ from $\Delta^N$, then make a decision  $\e_i \in \cD$ with probability $p_{t,i}$.
For such a scheme, the  regret guarantee is the same as in  the standard LEA, but with \emph{expected} regret.

Recall that each black-box run $\cB_J$ is on a different interval $J$.
The meta algorithm's role is to hedge  bets over multiple black-box runs.
Thus, it is natural to treat each run $\cB_J$ as an \emph{expert} and use an existing LEA algorithm to combine decisions from each expert $\cB_J$.
The loss incurred on run $\cB_J$ is $\ell_{t,\cB_J} := f_t(\x^{\cB_J}_t)$.

The challenge is that each expert $\cB_J$ may not output decisions at time steps outside the interval $J$.
This problem can be reduced to the \emph{sleeping experts} problem studied in~\cite{blum97empirical,freund97using}, where experts are not required to provide decisions at every time step; see~\cite{luo15achieving} for detail. 
We introduce a binary indicator variable $\cI_{t,i} \in \{0,1\}$, which is set to $1$ if expert $i$ is awake (that is, outputting a decision) at time $t$, and zero otherwise.
Define  $\bfcI_t := [\cI_{t,1},\cI_{t,2},\ldots,\cI_{t,N} ]^\T$ where $N$ can be countably infinite.
Note that the algorithm is aware of $\bfcI_t$ and must assign zero weight to the experts that are sleeping: $\cI_{t,i} = 0 \implies p_{t,i} = 0$.
We would like to have a guarantee on the regret w.r.t. expert $i$, but only for the time steps where expert $i$ is awake.
Following~\cite{luo15achieving}, we aim to have a regret bound w.r.t. $\u\in\Delta^N$ as follows:
\vspace{-5pt}
\begin{equation}\label{regret-sleeping}
  \Regret_T(\u) := \sum_{t=1}^T\sum_{i=1}^N \cI_{t,i} u_i(\langle \bfell_t,\p_t\rangle - \ell_{t,i}) \;.
\end{equation}
If we set $\u=\e_j$ for some $j$, the above is simply regret w.r.t. expert $j$ while that expert is awake.
Furthermore, if $\cI_{t,j} =1 $ for all $t \in [T]$, then it recovers the standard static regret in LEA.

\vspace{-5pt}
\paragraph{Coin Betting for LEA.}
We consider the coin betting framework of Orabona and P\'{a}l \cite{orabona16from}, where one can construct an LEA algorithm based on the so-called \emph{coin betting potential} function $F_t$.
A player starts from the initial endowment $1$.
At each time step, the adversary tosses a coin arbitrarily, with the player deciding upon which side to bet (heads or tails). Then the outcome is revealed.
The adversary can manipulate the weight of the coin in $[0,1]$ as well, in a manner not known to the player before betting.

We encode a coin flip at iteration $t$ as $\tilg_t\in[-1,1]$ where positive (negative) means heads (tails), and $|\tilg_t|$ indicates the weight.
Let $\Wealth_{t-1}$ be the total money the player possesses after time step $t-1$. 
The player decides which side and how much money to bet.
We encode the player's decision as the signed betting fraction $\beta_t \in (-1,1)$, where the positive (negative) sign indicates head (tail) and the absolute value $|\beta_t| \in [0,1)$ indicates the fraction of his money to bet.
Thus, the actual amount of betting is $w_t := \beta_t \Wealth_{t-1}$.
Once the weighted coin flip $\tilg_t$ is revealed, the player's wealth changes: $\Wealth_t = \Wealth_{t-1} + \tilg_t \beta_t \Wealth_{t-1}$.
The player makes (loses) money when the betted side is correct (wrong), and the amount of wealth change depends on both the flip weight $|\tilg_t|$ and his betting amount $|\beta_t|$.

In the coin betting framework, the betting fraction $\beta_t$ is determined by a potential function $F_t$, and we can simplify $w_t$ as follows:
\vspace{-7pt}
\begin{align}
z_t & := \sum_{\tau=1}^{t-1} \tilg_\tau \\
  \beta_t(z_t) &:= \fr{F_t(z_t + 1) - F_t(z_t - 1)}{F_t(z_t + 1) + F_t(z_t - 1)} \label{betting-fraction}\\
  w_t &= \beta_t(z_t) \cdot \lt( 1 + \sum_{\tau=1}^{t-1} \tilg_\tau w_\tau \rt) \label{betting-amount} \;.
\end{align}
We use $\beta_t$ in place of $\beta_t(\sum_{\tau=1}^{t-1} \tilg_\tau)$ when it is clear from the context.
A sequence of coin betting potentials $F_1, F_2, \ldots$ satisfies
the following key condition (the complete list of conditions can be found in~\cite{orabona16from}): $F_t$ must lower-bound the wealth of a player who bets by~\eqref{betting-fraction}:
\vspace{-8pt}
\begin{align}\label{wealth-lowerbound} 
 \forall t,\; F_t\left(\sum_{\tau=1}^t \tilg_\tau\right) \le 1 + \sum_{\tau=1}^{t} \tilg_\tau w_{\tau} \;.
\end{align}
This bound becomes useful in regret analysis.
We emphasize that the term $w_t$ is decided before $\tilg_t$ is revealed, yet the inequality  \eqref{wealth-lowerbound} holds for any $\tilg_t \in [-1,1]$.

Orabona and P\'{a}l~\cite{orabona16from} have devised a reduction of LEA to the simple coin betting problem described above.
The idea is to instantiate a coin betting problem for each expert $i$ where the signed coin flip $\tilg_{t,i}$ is set as a conditionally truncated regret w.r.t. expert $i$, rather than being set by an adversary.
We denote by $\beta_{t,i}$ the betting fraction for expert $i$ and by $w_{t,i}$ the amount of betting for expert $i$, $\forall i\in[N]$.

We apply the same treatment under the sleeping experts setting and propose a new algorithm \textbf{Sleeping CB}.
Since some experts may not output a decision at time $t$, Sleeping CB requires a different definition of $\beta_{t}$.
We define $S_{t,i} := 1+\sum_{\tau=1}^{t-1} \cI_{\tau,i}$ and define the following modifications of \eqref{betting-fraction}
\vspace{-5pt}
\begin{align*}
  z_{t,i}        &:= \sum_{\tau=1}^{t-1} \cI_{t,i}\tilg_{\tau,i} \\
  \beta_{t,i}(z_{t,i}) &:= \fr{F_{S_{t,i}}(z_{t,i} + 1) - F_{S_{t,i}}(z_{t,i} - 1) }{F_{S_{t,i}}(z_{t,i} + 1) + F_{S_{t,i}}(z_{t,i} - 1)} \;.
\end{align*}
Further, we denote by $\bfpi_{\bfcI_t}$ the prior $\bfpi$ restricted to experts that are awake ($\cI_{t,i}=1$), and define $[x]_+ := \max\{x, 0\}$.
Algorithm~\ref{alg:cblea-sleeping} specifies the Sleeping CB algorithm.

\begin{algorithm}[t]
{\small
\begin{algorithmic}
  \STATE \textbf{Input}: Number of experts $N$, prior distribution $\bfpi \in \Delta^{N}$
  \vspace{-7pt}
  \FOR {$t=1,2,\ldots$}
    \STATE For each $i\in \text{Active}(t)$, set \\
    \quad $w_{t,i} \larrow \beta_{t,i}(z_{t,i}) \cdot (1 + \sum_{\tau=1}^{t-1} \cI_{\tau,i} \tilg_{\tau,i} w_{\tau,i}) $. 
    \STATE For each $i\in \text{Active}(t)$, set $\hatp_{t,i} \larrow   \pi_i \cI_{t,i} [w_{t,i}]_+$.
    \STATE Predict with $\p_t \larrow \begin{cases} 
      \hatbfp_t / ||\hatbfp_t||_1 & \text{ if } ||\hatbfp_t||_1 > 0 \\
      \bfpi_{\bfcI_t}             & \text{ if } ||\hatbfp_t||_1 = 0. 
    \end{cases}  $
    \STATE Receive loss vector $\bfell_t \in [0,1]^N$.
    \STATE The learner suffers loss $h_t(\p_t) = \la \bfell_t, \p_t \ra_{\bfcI_t}$.
    \STATE For each $i \in \text{Active}(t)$, set \\
    \qquad$\tilg_{t,i} \larrow \begin{cases}
          h_t(\p_t) - \ell_{t,i}      & \text{ if } w_{t,i} > 0 \\
          [h_t(\p_t) - \ell_{t,i}]_+  & \text{ if } w_{t,i} \le 0.
        \end{cases}$
  \ENDFOR
\end{algorithmic}
\caption{Sleeping CB} 
\label{alg:cblea-sleeping}
}
\end{algorithm}

 The regret of Sleeping CB is bounded in Theorem~\ref{thm:cblea-sleeping}.
(All  proofs appear as supplementary material.)
Unlike the standard CB, in which all the experts use $F_t$ at time $t$, expert $i$ in Sleeping CB uses $F_{S_{t,i}}$, which is different for each expert.
For this reason, the proof of the CB regret in~\cite{orabona16from} does not transfer easily to the regret~\eqref{regret-sleeping} of Sleeping CB, and a solution to it is the cornerstone of an improved strongly adaptive regret bound.
\begin{thm} \label{thm:cblea-sleeping}
  \emph{(Regret of Sleeping CB)}
  Let $\{F_t\}_{t\ge1}$ be a sequence of potential functions that satisfies~\eqref{wealth-lowerbound}.
  Assume that $F_t$ is even (symmetric around zero) $\forall t \ge 1$.
  Suppose $\log F_{S_{T,i}} (x) \ge h_{S_{T,i}}(x) := c_1 \fr{x^2}{S_{T,i}} + c_{2,i}$ for some $c_1>0$ and $c_{2,i}\in\dsR$ for all $i\in[N]$.
  Then, Algorithm~\ref{alg:cblea-sleeping} satisfies
  \begin{align*}
 &\emph{\text{Regret}}_T(\u)  \\
 &\le \sqrt{ c_1^{-1} \cdot \lt( \sum_{i=1}^N u_i S_{T,i} \rt) \cdot \lt( \text{\emph{KL}}(\u || \bfpi) - \sum_{i=1}^N u_i c_{2,i} \rt) } \;. 
\end{align*}
\end{thm}
Note that if $\u=\e_j$, then the regret scales with $S_{T,j}$, which is essentially the number of time steps at which expert $j$ is awake.

Among multiple choices for the potential, we use the Krichevsky-Trofimov (KT) potential~\cite{orabona16from} that satisfies~\eqref{wealth-lowerbound} (see~\cite{orabona16from} for the proof):
\begin{align}
  F_t(x) = \fr{2^t\cdot\Gamma(\dt+1)\cdot\Gamma(\fr{t+\dt+1}{2}+\fr{x}{2})\Gamma(\fr{t+\dt+1}{2}-\fr{x}{2})}{\Gamma(\fr{\dt+1}{2})^2\cdot\Gamma(t+\dt+1)}, 
\end{align}
where $\dt \ge 0$ is a time shift parameter that we set to 0 in this work.
One can show that the betting fraction $\beta_t$ defined in~\eqref{betting-fraction} for KT potential exhibits a simpler form: $\beta_t = \fr{\sum_{\tau=1}^{t-1} \tilg_{\tau}}{t+\dt}$~\cite{orabona16from} and, for Sleeping CB, $\beta_t = \fr{\sum_{\tau=1}^{t-1}\cI_{\tau,i}\tilg_{\tau,i}}{S_{t,i} + \dt} $.
We present the regret of Algorithm~\ref{alg:cblea-sleeping} with the KT potential in Corollary~\ref{cor:cblea-sleeping}. 
\begin{cor} \label{cor:cblea-sleeping}
  Let $\dt=0$.
  The regret of Algorithm~\ref{alg:cblea-sleeping} with the KT potential is
  \begin{align*}
  &\text{\emph{Regret}}_T(\u)  \\
  &\le \sqrt{2\lt(\sum_{i=1}^N u_i S_{T,i} \rt)\cdot \lt( \text{\emph{KL}}(\u || \bfpi) + \fr{1}{2} \ln(T) + 2 \rt) } \;.
  \end{align*}
\end{cor}

\vspace{-5pt}
\section{Coping with a Changing Environment by Sleeping CB}
\label{sec:cbce}
\vspace{-5pt}
In this section, we synthesize the results in Sections~\ref{sec:meta} and~\ref{sec:coinbetting} to specify and analyze our algorithm.
Recall that a meta algorithm must efficiently aggregate decisions from multiple black-box runs that are active at time $t$.
We treat each black-box run as an expert.
Since we run a black-box instance for each interval $J\in\cJ$, there are a countably infinite number of experts.
Thus, one can use Sleeping CB (Algorithm~\ref{alg:cblea-sleeping}) as the meta algorithm, with  geometric covering intervals. 
The complete algorithm, which we call   \textbf{Coin Betting for Changing Environment (CBCE)}, is shown in Algorithm~\ref{alg:cbce}.

\begin{algorithm}[t]
{\small
\begin{algorithmic}
  \STATE \textbf{Input}: A black-box algorithm $\cB$ and a prior distribution $\bfpi \in \Delta^{|\cJ|}$ over $\{\cB_J \mid J \in \cJ\}$. 
  \FOR {$t = 1$ \TO $T$}
  \STATE For each $J\in\Active(t)$, set \\
  \quad $w_{t,\cB_J} \larrow \beta_{t,\cB_J}(z_{t,\cB_J})\cdot(1 + \sum_{\tau=1}^{t-1} \cI_{\tau,\cB_J} \tilg_{\tau,\cB_J} w_{\tau,\cB_J}) $
  \STATE Set $\hatp_{t,\cB_J} \larrow \pi_{\cB_J} \cI_{t,\cB_J}[w_{t,\cB_J}]_+$ for $J\in \text{Active}(t)$ and 0 for $J\not\in\Active(t)$.
  \STATE Compute $\p_t \larrow \begin{cases} 
    \hatbfp_t / ||\hatbfp_t||_1 & \text{ if } ||\hatbfp_t||_1 > 0 \\
    [\pi_{\cB_J}]_{J \in \Active(t)}  & \text{ if } ||\hatbfp_t||_1 = 0. 
  \end{cases}  $
  \STATE The black-box run $\cB_J$ picks a decision $\x^{\cB_J}_t \in \cD$, $\forall J \in \text{Active}(t)$.
  \STATE The learner picks a decision $\x_t = \sum_{J \in \cJ} p_{t,\cB_J} \x^{\cB_J}_t $.
  \STATE Each black-box run $\cB_J$ that is awake ($J\in\text{Active}(t)$) suffers loss $\ell_{t,\cB_J} := f_t(\x^{\cB_J}_t)$.
  \STATE The learner suffers loss $f_t(\x_t)$. 
  \STATE For each $J \in \text{Active}(t)$, set \\
  \qquad $\tilg_{t,\cB_J} \larrow \begin{cases}
    f_t(\x_t) - \ell_{t,\cB_J}  & \text{ if } w_{t,\cB_J} > 0 \\ 
    [f_t(\x_t) - \ell_{t,\cB_J}]_+ & \text{ if } w_{t,\cB_J} \le 0.
  \end{cases}$
  \ENDFOR
\end{algorithmic}
\caption{Coin Betting for Changing Environment (CBCE)}
\label{alg:cbce}
}
\end{algorithm}
We make use of the following assumption.
\begin{assump}
  \label{ass:convexity}
  The loss function $f_t$ is convex and maps to $[0,1]$,  $\forall t\in\dsN$.
\end{assump}
Nonconvex loss functions can be accommodated by randomized decisions: We  choose the decision $\x_t^{\cB_J}$ from black-box $\cB_J$ with probability $p_{t,\cB_J}$.
It is not difficult to  show that the same regret bound holds, but now in expectation.
When loss functions are unbounded, they can be scaled and restricted to $[0,1]$.
Although this leads to possible nonconvexity, we can still obtain an expected regret bound from the randomized decision  process just described.

We define our choice of prior $\bar\bfpi \in \Delta^{|\cJ|}$ as follows:
\begin{align}\label{def:barpi}
  \bar \pi_{\cB_J} := Z^{-1} \lt( \fr{1}{J_1^2 (1 + \lfl\log_2 J_1\rfl)} \rt) ,\; \forall J \in \cJ \;,
\end{align}
where $Z$ is a normalization factor.
Note that $Z < \pi^2/6$ since there exist at most $1 + \lfl \log_2 J_1 \rfl$ distinct intervals starting at time $J_1$, so $Z$ is less than $\sum_{t=1}^\infty t^{-2} = \pi^2/6$.

We bound the  meta regret w.r.t. a black-box run $\cB_J$  as follows.
\begin{lem} \label{lem:cbce}
  \emph{(Meta regret of CBCE)}
  Assume~\ref{ass:convexity}.
Suppose we run CBCE (Algorithm~\ref{alg:cbce}) with a black-box algorithm $\cB$, prior $\bar\bfpi$, and $\dt = 0$.
  The meta regret of CBCE$\la\cB\ra$ on interval $J=[J_1..J_2]\in\cJ$ is
  \begin{align*}
  &\sum_{t\in J} f_t(\x^{\emph{\CBCE}\la\cB\ra}_t) - f_t(\x^{\cB_J}_t)  \\
  &\le \sqrt{|J| \lt( 7\ln(J_2) + 5\rt) }
   = O(\sqrt{|J|\log J_2}) \;.
  \end{align*}
\end{lem}

We now present the bound on the SA-Regret $R^{\text{\CBCE}\la\cB\ra}_I (\w)$ w.r.t. $\w \in\cW$ on intervals $I\subseteq [T]$ that are not necessarily in $\cJ$.
\vspace{.5em}
\begin{thm}\label{thm:untitled}
  \emph{(SA-Regret of \texorpdfstring{$\CBCE\la\cB\ra$}{})}
  Assume~\ref{ass:convexity} and that the black-box algorithm $\cB$ has regret $R^{\cB}_T$ bounded by $c T^\alpha$, where $\alpha \in (0,1)$.
  Let $I = [I_1.. I_2]$.
  The SA-Regret of $\CBCE$ with black-box $\cB$ on the interval $I$ w.r.t. any $\w \in \cW$ is bounded as follows:
\begin{align*}
  R^{\emph{\CBCE}\la\cB\ra}_I (\w) 
  &\le \fr{4}{2^\alpha - 1} c|I|^\alpha + 8\sqrt{|I| (7\ln ( I_2) + 5)}\\
  &= O ( c |I|^\alpha + \sqrt{|I| \ln I_2} ) \;.
\end{align*}
\end{thm}
For the standard LEA problem, one can run the algorithm CB with KT potential (equivalent to Sleeping CB with $\cI_{t,i}=1, \forall t,i$), which achieves static regret $O(\sqrt{ T \log (NT) })$~\cite{orabona16from}.
Using CB as the black-box algorithm, the regret of $\CBCE\la\cB\ra$ on $I$ is $R^{\CBCE\la\CB\ra}_I(\w) = O(\sqrt{ |I| \log (N I_2) })$, and so SA-Regret$^{\CBCE\la\CB\ra}_T(|I|) =O(\sqrt{ |I| \log (N T) }) $.
It follows that the $m$-shift regret of $\CBCE\la\CB\ra$ is $O(\sqrt{mT\log(NT)})$ using the technique presented our supplementary material. 

As said above, our bound improves over the best known result with the same time complexity in~\cite{daniely15strongly}.
The key ingredient that allows us to get a better bound is the Sleeping CB Algorithm~\ref{alg:cblea-sleeping}, that achieves a better SA-Regret than the one of \cite{daniely15strongly}. In the next section, we will show that the empirical results also confirm the theoretical gap of these two algorithms.

\paragraph{Discussion.}
Note that one can obtain the same result using the data streaming intervals (DS)~\cite{hazan07adaptive,gyorgy12efficient} in place of the geometric covering intervals (GC).
Section~\ref{sec:ds} of our supplementary material elaborates on this with a Lemma stating that DS induces a partition of an interval $I$ in a very similar way to GC (a sequence of intervals of doubling lengths).

Our improved bound has another interesting implication.
In designing strongly adaptive algorithms for LEA, there is a well known technique called ``restarts'' or ``sleeping experts'' that has time complexity $O(NT^2)$~\cite{hazan07adaptive,luo15achieving}, and several studies used DS or GC to reduce the time complexity to $O(NT\log T)$~\cite{hazan07adaptive,gyorgy12efficient,daniely15strongly}.
However, it was unclear whether it is possible to achieve both an $m$-shift regret of $O(\sqrt{mT(\log N + \log T)})$ and a time complexity of $O(NT\log T)$ without knowing $m$.
Indeed, every study on $m$-shift regret with time $O(NT\log T)$ results in suboptimal $m$-shift regret bounds~\cite{daniely15strongly,gyorgy12efficient,hazan07adaptive}, to our knowledge.
Furthermore, some studies (e.g.,~\cite[Section~5]{luo15achieving}) speculated that perhaps applying the data streaming technique would increase its SA-Regret by a logarithmic factor.
Our analysis implies that one can reduce the overall time complexity to $O(NT\log T)$ without sacrificing the order of SA-Regret and $m$-shift regret. %

%

\vspace{-7pt}
\section{Experiments}
\label{sec:expr}
\vspace{-7pt}

We now turn to an empirical evaluation of algorithms for changing environments.
We compare the performance of the meta algorithms under two online learning problems: $(i)$ learning with expert advice (LEA) and $(ii)$ metric learning (ML).
We compare CBCE with SAOL~\cite{daniely15strongly} and AdaNormalHedge.TV (ATV)~\cite{luo15achieving}.
Although ATV was originally designed for LEA only, it is not hard to extend it to a meta algorithm and show that it has the same order of SA-Regret as CBCE using the same techniques.

For our empirical study, we replace the geometric covering intervals (GC) with the data streaming intervals (DS)~\cite{hazan07adaptive,gyorgy12efficient}. 
Let $u(t)$ be a number such that $2^{u(t)}$ is the largest power of 2 that divides $t$; e.g., $u(12)=2$.
The data streaming intervals are $\cJ = \{[t..(t+g\cdot2^{u(t)}-1)]: t =1,2,\ldots \}$ for some $g \ge 1$.
DS is an attractive alternative, unlike GC, $(i)$ DS initiates one and only one black-box run at each time, and $(ii)$ it is more flexible in that the parameter $g$ can be increased to enjoy smaller regret in practice while increasing the time complexity by a constant factor.

For both ATV and CBCE, we set the prior $\bfpi$ over the black-box runs as the uniform distribution. 
Note that this does not break the theoretical guarantees since the number of black-box runs are never actually infinite; we used $\bar\bfpi$~\eqref{def:barpi} in Section~\ref{sec:cbce} for ease of exposition.

\begin{figure*}
  {\centering
  \begin{tabular}{c}
    \includegraphics[width=.9\textwidth,valign=t]{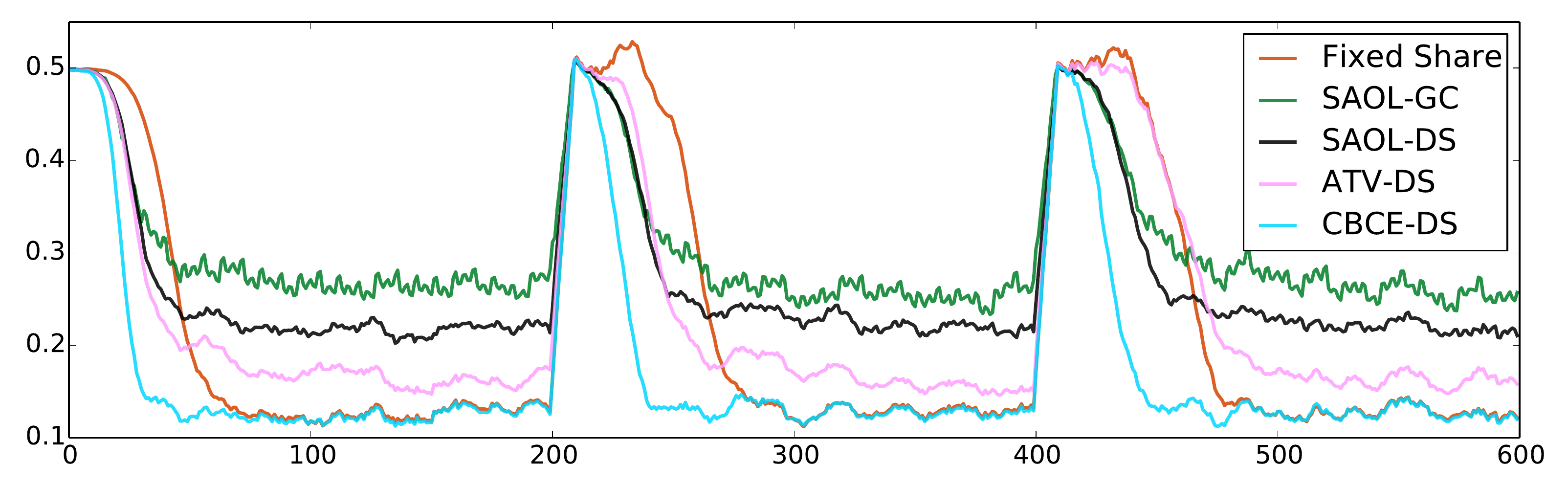} \\ (a) Learning with expert advice \vspace{-2pt}\\ 
    \includegraphics[width=.9\textwidth,valign=t]{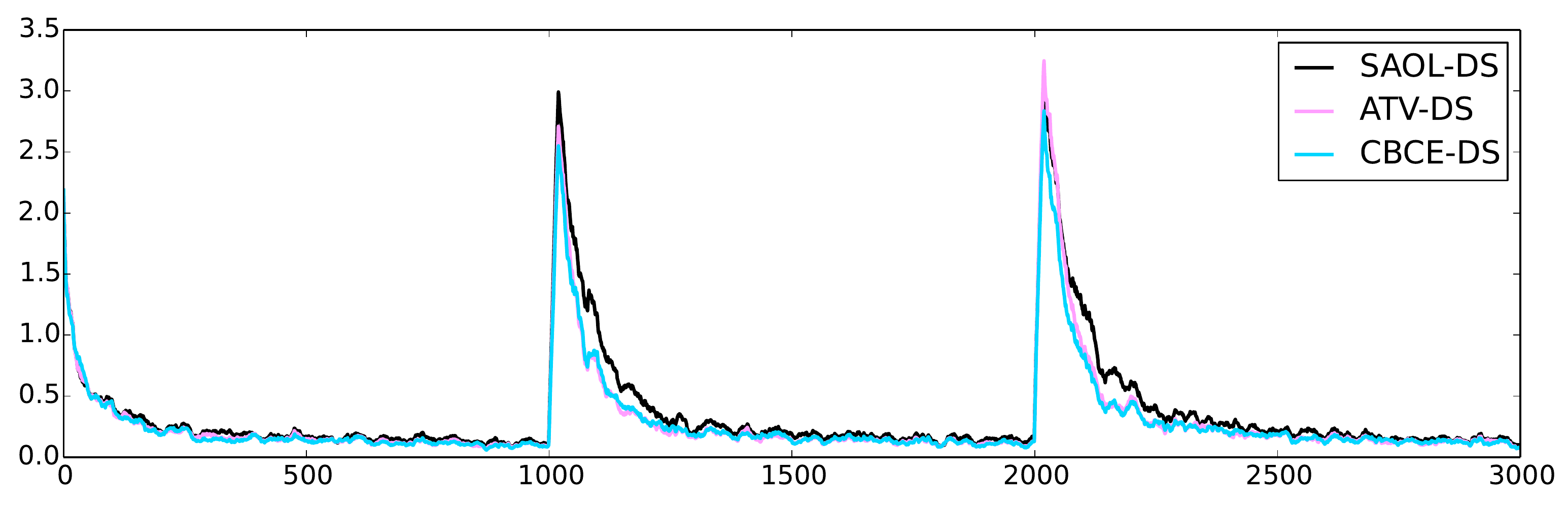} \\ (b) Metric learning
  \end{tabular}
  \vspace{-9pt}
  \caption{Experiment results: Our method CBCE outperforms several baseline methods. }
  \label{fig:expr}
}
\vspace{-15pt}
\end{figure*}

\vspace{-5pt}
\subsection{Learning with Expert Advice (LEA)}
\vspace{-5pt}

We consider LEA with linear loss.
That is, the loss function at time $t$ is $h_t(\p) = \l_t^\T \p$.
We draw linear loss $\l_t \in [0,1]^{N}, \forall t=1,\ldots,600$ for $N=1000$ experts from Uniform$(0,1)$ distribution.
Then, for time $t\in[1,200]$, we reduce loss of expert 1 by subtracting 1/2 from its loss: $\ell_{t,1} \larrow [\ell_{t,1} - 1/2]_+$.
For time $t\in[201,400]$ and $t\in[401,600]$, we perform the same for expert 2 and 3, respectively.
Thus, the best expert is 1, 2, and 3 for time segment [1,200], [201,400], and [401,600], respectively.
We use the data streaming intervals with $g=2$.
In all our experiments, DS with $g=2$ outperforms GC while spending roughly the same time.

For each meta algorithm, we use the CB with KT potential~\cite{orabona16from} as the black-box algorithm.
We warm-start each black-box run at time $t\ge2$ by setting its prior to the decision $\p_{t-1}$ chosen by the meta algorithm at time step $t-1$.
We repeat the experiment 50 times and plot their average loss by computing moving mean with window size 10 in Figure~\ref{fig:expr}(a).
Overall, we observe that CBCE (i) catches up with the environmental shift faster than any other algorithms and (ii) has the lowest loss when the shift has settled down.
ATV is the second best, outperforming SAOL.
Note that SAOL with GC (SAOL-GC) tends to incur larger loss than the SAOL with DS.
We observe that this is true for every meta algorithm, so we omit the result here to avoid clutter.
We also run Fixed Share using the parameters recommended by Corollary 5.1 of~\cite{cesa-bianchi06prediction}, which requires to know the target time horizon $T=600$ and the true number of switches $m=2$.
Such a strong assumption is often unrealistic in practice.
We observe that Fixed Share is the slowest in adapting to the environmental changes. 
Nevertheless, Fixed Share remains attractive since (i) after the switch has settled down its loss is competitive to CBCE, and (ii) its time complexity is lower than other algorithms ($O(NT)$ rather than $O(NT\log T)$).

\vspace{-5pt}
\subsection{Metric Learning}
\vspace{-5pt}

We consider the problem of learning squared Mahalanobis distance from pairwise comparisons using the mirror descent algorithm~\cite{kunapuli12mirror}.
The data point at time $t$ is $(\z^{(1)}_t, \z^{(2)}_t, y_t)$, where 
$y_t \in\{1,-1\}$ indicates whether or not $\z^{(1)}_t \in \dsR^d$ and $\z^{(2)}_t \in \dsR^d$ belongs to the same class.
The goal is to learn a squared Mahalanobis distance parameterized by a positive semi-definite matrix $\M$ and a bias $\mu$ that have small loss $f_t([\M; \mu]) :=$
\vspace{-5pt}
\begin{align*}
  [1 - y_t(\mu - (\z^{(1)}_t - \z^{(2)}_t)^\T \M (\z^{(1)}_t - \z^{(2)}_t))]_+ + \rho || \M ||_* \;,
\end{align*}
where $\mu$ is the bias parameter and $|| \cdot ||_*$ is the trace norm.
Such a formulation encourages predicting $y_t$ with large margin and low rank in $\M$.
A learned matrix $\M$ that has low rank can be useful in a number of machine learning tasks; e.g., distance-based classifications, clusterings, and low-dimensional embeddings.
We refer to~\cite{kunapuli12mirror} for details.

We create a scenario that exhibits shifts in the metric, which is inspired by~\cite{greenewald16nonstationary}. 
Specifically, we create a mixture of three Gaussians in $\dsR^3$ whose means are well-separated, and mixture weights are .5, .3, and .2.
We draw 2000 points from it while keeping a record of their memberships. 
We repeat this three times independently and concatenate these three vectors to have 2000 9-dimensional vectors. 
Finally, we append to each point a 16-dimensional vector filled with Gaussian noise to have 25-dimensional vectors. 
Such a construction implies that for each point there are three independent cluster memberships.
We run each algorithm for 1500 time steps.
For time 1 to 500, we randomly pick a pair of points from the data pool and assign $y_t=1$ $(y_t=-1)$ if the pair belongs to the same (different) cluster under the first clustering.
For time 501 to 1000 (1001 to 1500), we perform the same but under the second (third) clustering.
In this way, a learner faces tracking the change in metric, especially the important low-dimensional subspaces for each time segment.

Since the loss of the metric learning is unbounded, we scale the loss by multiplying 1/5 and then capping it above at 1 as in~\cite{greenewald16nonstationary}.
Although the randomized decision discussed in Section~\ref{sec:cbce} can be used to maintain the theoretical guarantee, we stick to the weighted average since the event that the loss being capped at 1 is rare in our experiments.
As in our LEA experiment, we use the data streaming intervals with $g=2$ and initialize each black-box algorithm with the decision of the meta algorithm at the previous time step.
We repeat the experiment 50 times and plot their average loss in Figure~\ref{fig:expr}(b) by moving mean with window size 20.
We observe that CBCE and ATV both outperforms SAOL.
This confirms the improved regret bound of CBCE and ATV.

\vspace{-10pt}
\section{Future Work}
\label{sec:future}
\vspace{-9pt}

Among a number of interesting directions, we are interested in reducing the time complexity in the online learning within a changing environment.
For LEA, Fixed Share has the best time complexity.
However, Fixed Share is inherently not parameter-free; especially, it requires the knowledge of the number of shifts $m$. 
Achieving the best $m$-shift regret bound without knowing $m$ or the best SA-Regret bound in time $O(NT)$ would be an interesting future work.
The same direction is interesting for the online convex optimization (OCO) problem.
It would be interesting if an OCO algorithm such as online gradient descent can have the same SA-Regret as CBCE$\la$OGD$\ra$ without paying extra order of computation.

\vspace{-10pt}
\subsubsection*{Acknowledgements}
\vspace{-5pt}
This work was supported by NSF Award IIS-1447449 and NIH Award 1 U54 AI117924-01.
The authors thank Andr\'as Gy\"orgy for providing constructive feedback and Kristjan Greenewald for providing the metric learning code.

\newpage
\bibliographystyle{IEEEannot}
\bibliography{library-shared}

\begin{thebibliography}{10}
\providecommand{\url}[1]{#1}
\csname url@rmstyle\endcsname
\providecommand{\newblock}{\relax}
\providecommand{\bibinfo}[2]{#2}
\providecommand\BIBentrySTDinterwordspacing{\spaceskip=0pt\relax}
\providecommand\BIBentryALTinterwordstretchfactor{4}
\providecommand\BIBentryALTinterwordspacing{\spaceskip=\fontdimen2\font plus
\BIBentryALTinterwordstretchfactor\fontdimen3\font minus
  \fontdimen4\font\relax}
\providecommand\BIBforeignlanguage[2]{{%
\expandafter\ifx\csname l@#1\endcsname\relax
\typeout{** WARNING: IEEEtran.bst: No hyphenation pattern has been}%
\typeout{** loaded for the language `#1'. Using the pattern for}%
\typeout{** the default language instead.}%
\else
\language=\csname l@#1\endcsname
\fi
#2}}

\bibitem{adamskiy12acloser}
D.~Adamskiy, W.~M. Koolen, A.~Chernov, and V.~Vovk, ``{A Closer Look at
  Adaptive Regret},'' in \emph{Proceedings of the International Conference on
  Algorithmic Learning Theory (ALT)}, 2012, pp. 290--304.


\bibitem{blum97empirical}
A.~Blum and A.~Blum, ``{Empirical Support for Winnow and Weighted-Majority
  Algorithms: Results on a Calendar Scheduling Domain},'' \emph{Machine
  Learning}, vol.~26, no.~1, pp. 5--23, 1997.


\bibitem{cesa-bianchi12mirror}
N.~Cesa-Bianchi, P.~Gaillard, G.~Lugosi, and G.~Stoltz, ``{Mirror descent meets
  fixed share (and feels no regret)},'' in \emph{Advances in Neural Information
  Processing Systems (NIPS)}, 2012, pp. 980--988.


\bibitem{cesa-bianchi06prediction}
N.~Cesa-Bianchi and G.~Lugosi, \emph{{Prediction, Learning, and Games}}.\hskip
  1em plus 0.5em minus 0.4em\relax Cambridge University Press, 2006.


\bibitem{daniely15strongly}
A.~Daniely, A.~Gonen, and S.~Shalev-Shwartz, ``{Strongly Adaptive Online
  Learning},'' \emph{Proceedings of the International Conference on Machine
  Learning (ICML)}, pp. 1--18, 2015.


\bibitem{freund97using}
Y.~Freund, R.~E. Schapire, Y.~Singer, and M.~K. Warmuth, ``{Using and combining
  predictors that specialize},'' \emph{Proceedings of the ACM symposium on
  Theory of computing (STOC)}, vol.~37, no.~3, pp. 334--343, 1997.


\bibitem{greenewald16nonstationary}
K.~Greenewald, S.~Kelley, and A.~O. Hero, ``{Dynamic metric learning from
  pairwise comparisons},'' \emph{54th Annual Allerton Conference on
  Communication, Control, and Computing (Allerton)}, 2016.


\bibitem{gyorgy12efficient}
A.~Gy{\"{o}}rgy, T.~Linder, and G.~Lugosi, ``{Efficient tracking of large
  classes of experts},'' \emph{IEEE Transactions on Information Theory},
  vol.~58, no.~11, pp. 6709--6725, 2012.


\bibitem{hazan07adaptive}
E.~Hazan and C.~Seshadhri, ``{Adaptive Algorithms for Online Decision
  Problems},'' \emph{IBM Research Report}, vol. 10418, pp. 1--19, 2007.


\bibitem{herbster98tracking}
M.~Herbster and M.~K. Warmuth, ``{Tracking the Best Expert},'' \emph{Mach.
  Learn.}, vol.~32, no.~2, pp. 151--178, 1998.


\bibitem{kunapuli12mirror}
G.~Kunapuli and J.~Shavlik, ``{Mirror descent for metric learning: A unified
  approach},'' in \emph{Proceedings of the European Conference on Machine
  Learning and Principles and Practice of Knowledge Discovery in Database
  (ECML/PKDD)}, 2012, pp. 859--874.


\bibitem{luo15achieving}
H.~Luo and R.~E. Schapire, ``{Achieving All with No Parameters:
  AdaNormalHedge},'' in \emph{Proceedings of the Conference on Learning Theory
  (COLT)}, 2015, pp. 1286--1304.


\bibitem{orabona16from}
F.~Orabona and D.~Pal, ``Coin betting and parameter-free online learning,'' in
  \emph{Advances in Neural Information Processing Systems (NIPS)}, 2016, pp.
  577--585.


\bibitem{veness13partition}
J.~Veness, M.~White, M.~Bowling, and A.~Gy\"{o}rgy, ``Partition tree
  weighting,'' in \emph{Proceedings of the 2013 Data Compression
  Conference}.\hskip 1em plus 0.5em minus 0.4em\relax IEEE Computer Society,
  2013, pp. 321--330.


\end{thebibliography}

\clearpage

\onecolumn
\begin{center}
{\Large\bf Supplementary Material}
\end{center}
\pdfoutput=1 
\renewcommand\thesection{\Alph{section}}
\setcounter{section}{0}

\section{Strongly Adaptive Regret to \texorpdfstring{$m$}{}-Shift Regret}
\label{sec:regret-conversion}
We present an example where a strongly adaptive regret bound can be turned into an $m$-shift regret bound. 

Let $c>0$.
We claim that: 
$$
\lt( \forall I=[I_1..I_2],\; R^\cA_I(\w) \le c\sqrt{ |I| \log (I_2)} \rt) \implies m\mbox{-Shift-Regret}_T^\cA \le c\sqrt{(m+1)T\log(T)} \;.
$$
To prove the claim, note that an $m$-shift sequence of experts $\w_{1:T}$ can be partitioned into $m+1$ contiguous blocks denoted by $I^{(1)}, \ldots, I^{(m+1)}$; e.g., $(1,1,2,2,1)$ is 2-switch sequence whose partition $\{ [1,2], [3,4], [5] \}$.
Denote by $\w_{I^{(k)}} \in \cW$ the comparator in interval $I^{(k)}$: $ \w_t = \w_{I^{(k)}}, \forall t\in I^{(k)}$. 
Then, using Cauchy-Schwartz inequality, 
\begin{align} \label{regret-conversion}
&  m\text{-Shift-Regret}_T^\cA  \notag\\
&= \max_{\w_{1:T}: m\text{-shift seq.}} \sum_{k=1}^{m+1} R^\cA_{I^{(k)}} ( \w_{I^{(k)}} )  \notag \\
&\le \max_{\w_{1:T}: m\text{-shift seq.}} c\sqrt{\log T} \sum_{k=1}^{m+1} \sqrt{|I^{(k)}|} \notag\\
&\le \max_{\w_{1:T}: m\text{-shift seq.}} c\sqrt{\log T} \sqrt{(m+1)} \cdot \sqrt{\sum_{k=1}^{m+1} |I^{(k)}|} \notag \\
&= c\sqrt{\log T} \sqrt{(m+1)} \cdot \sqrt{T} \;.
\end{align}

\section{Proof of Theorem~\ref{thm:cblea-sleeping}}

\begin{proof}
  First, we show that $ \sum_{i=1}^N \pi_i \cI_{t,i} \tilg_{t,i} w_{t,i} \le 0$.
\begin{align*}
\sum_{i=1}^N   \pi_i \cI_{t,i}\tilg_{t,i} w_{t,i} 
&= \sum_{i:\pi_i \cI_{t,i} w_{t,i} > 0} \pi_i [w_{t,i}]_+ (\la \bfell_t,\p_t \ra_{\bfcI_t} - \ell_{t,i}) + \sum_{i: \pi_i\cI_{t,i} w_{t,i}\le 0} \pi_i \cI_{t,i} w_{t,i}[\la \bfell_t,\p_t \ra_{\bfcI_t} - \ell_{t,i}]_+ \\
&= ||\hatbfp_t||_1 \sum_{i:\pi_i\cI_{t,i} w_{t,i} > 0}  p_{t,i}(\la \bfell_t,\p_t \ra_{\bfcI_t} - \ell_{t,i}) + \sum_{i: \pi_i\cI_{t,i} w_{t,i}\le 0}  \pi_i\cI_{t,i} w_{t,i}[\la \bfell_t,\p_t \ra_{\bfcI_t} - \ell_{t,i}]_+ \\
&= 0 + \sum_{i:  \pi_i\cI_{t,i} w_{t,i}\le 0} \pi_i \cI_{t,i} w_{t,i}[\la \bfell_t,\p_t \ra_{\bfcI_t} - \ell_{t,i}]_+  
\le 0 \;.
\end{align*}
Then, due to the property of the coin betting potentials~\eqref{wealth-lowerbound},
\begin{align}\label{constantbound}
  \sum_{i=1}^N \pi_i F_{S_{T,i}}\lt( \sum_{t=1}^T \cI_{t,i} \tilg_{t,i} \rt) \le 1 + \sum_{i=1}^N \pi_i \sum_{t=1}^T \cI_{t,i} \tilg_{t,i} w_{t,i} \le 1 \;.
\end{align}
Define $\tilG_{T,i} := \sum_{t=1}^T \cI_{t,i}\tilg_{t,i}$. 
Since $F_T$ is even,
\begin{align*}
\sum_{i=1}^N u_i \log (F_{S_{T,i}}(|\tilG_{T,i}|))
&=\sum_{i=1}^N u_i \log (F_{S_{T,i}}(\tilG_{T,i})) \\
&\le \sum_{i=1}^N u_i \lt( \log\lt(\fr{u_i}{\pi_i}\rt) + \log\lt(\fr{\pi_i}{u_i} \cdot F_{S_{T,i}}(\tilG_{T,i}) \rt) \rt) \\
&\le \text{KL}(\u||\bfpi) + \sum_{i=1}^N u_i \log\lt(\fr{\pi_i}{u_i} \cdot F_{S_{T,i}}(\tilG_{T,i}) \rt)  \\
&\le \text{KL}(\u||\bfpi) + \log\lt(\sum_{i=1}^N u_i \cdot \fr{\pi_i}{u_i} \cdot F_{S_{T,i}}(\tilG_{T,i}) \rt)   \\
&\stackrel{\eqref{constantbound}}{\le} \text{KL}(\u||\bfpi) \;.
\end{align*}
Then, for any comparator $\u \in \Delta^N$,
\begin{align*}
\text{Regret}_T(\u) &= \sum_{t=1}^T\sum_{i=1}^N \cI_{t,i} u_i(\langle \bfell_t,\p_t\rangle - \ell_{t,i}) \\
&\le \sum_{t=1}^T \sum_{i=1}^N \cI_{t,i} u_i \tilg_{t,i} 
 \le \sum_{i=1}^N u_i |\tilG_{T,i}| \\
&= \sum_{i=1}^N u_i h^{-1}_{S_{T,i}} ( h_{S_{T,i}} (|\tilG_{T,i}|) ) \\
&\le \sum_{i=1}^N u_i h^{-1}_{S_{T,i}} ( \log (F_{S_{T,i}} (|\tilG_{T,i}|)) ) \\
&= \sum_{i=1}^N u_i \sqrt{c_1^{-1}\cdot (S_{T,i} )\cdot(\log (F_{S_{T,i}}(|\tilG_{T,i}|)) - c_{2,i}) } \\
&= \sum_{i=1}^N  \sqrt{c_1^{-1} u_i S_{T,i} }\cdot\sqrt{u_i(\log (F_{S_{T,i}}(|\tilG_{T,i}|)) - c_{2,i}) } \\
&\stackrel{(a)}{\le} \sqrt{ c_1^{-1} \lt( \sum_{i=1}^N u_i S_{T,i} \rt) \cdot \lt( \sum_{i=1}^N u_i(\log (F_{S_{T,i}}(|\tilG_{T,i}|)) - c_{2,i}) \rt) }  \\
&\le \sqrt{ c_1^{-1} \lt( \sum_{i=1}^N u_i S_{T,i} \rt) \cdot \lt( \text{KL}(\u || \bfpi) - \sum_{i=1}^N u_i c_{2,i} \rt) },
\end{align*}
where $(a)$ is due to the Cauchy-Schwarz inequality (verify that the factors under the square root are all nonnegative since $\log F_{S_{T,i}}(x) \ge h_{S_{T,i}}(x))$.

\end{proof}

\section{Proof of Corollary~\ref{cor:cblea-sleeping}}
\begin{proof}
  Define $h_{S_{T,i}}(x) := \fr{x^2}{2S_{T,i}} + \fr{1}{2} \ln(\fr{1}{S_{T,i}}) - \ln(e\sqrt{\pi}) $.
  According to Lemma 15 of~\cite{orabona16from} with $\dt=0$, $h_{S_{T,i}}(x) \le \ln F_{S_{T,i}}(x)$.
  Thus, from the context of Theorem~\ref{thm:cblea-sleeping}, $c_1 = 1/2$ and $c_2=\fr{1}{2}\ln(\fr{1}{S_{T,i}}) - \ln (e\sqrt{\pi})$.
  Then,
  \begin{align*}
  -\sum_{i=1}^N u_i c_{2,i}
  &= \sum_{i=1}^N u_i \lt( (1/2) \ln(S_{T,i}) + \ln(e\sqrt{\pi}) \rt) \\
  &\le \fr{1}{2} \ln(T) + 2 \;.
  \end{align*}
  Plugging in $c_1$ and $c_2$ to Theorem~\ref{thm:cblea-sleeping},
  \begin{align*}
  &\text{Regret}_T(\u)  \\
  &\le \sqrt{2\lt(\sum_{i=1}^N u_i S_{T,i} \rt)\cdot \lt( \text{{KL}}(\u || \bfpi) - \sum_{i=1}^N u_i c_{2,i} \rt)  } \\
  &\le \sqrt{2\lt(\sum_{i=1}^N u_i S_{T,i} \rt)\cdot \lt( \text{{KL}}(\u || \bfpi) + \fr{1}{2} \ln(T) + 2 \rt)  } \;.
  \end{align*}
\end{proof} 

\section{Proof of Lemma~\ref{lem:cbce}}
\label{sec:lemma}
\begin{proof}
  Note that our regret definition for meta algorithms 
  \begin{align}\label{pf-lem-cbce-1}
  \sum_{t\in J} f_t(\x_t^{\CBCE\la\cB\ra}) - f_t(\x_t^{\cB_J})\;,
  \end{align}
  is slightly different from that of Theorem~\ref{thm:cblea-sleeping} for $\u=\e_i$: $\sum_{t\in J:\cI_{t,i} = 1} \la\l_t,\p_t\ra - \ell_{t,i}$.
  This translates to, in the language of meta algorithms, $\sum_{t\in J:\cI_{t,\cB_J} = 1} \la\l_t,\p_t\ra_{\bfcI_t} - \ell_{t,\cB_J}$ for $\u=\e_{\cB_J}$ (recall $\ell_{t,\cB_J} = f_t(\x_t^{\cB_J})$).

  We claim that Theorem~\ref{thm:cblea-sleeping} and Corollary~\ref{cor:cblea-sleeping} for hold true for the regret~\eqref{pf-lem-cbce-1}.
  Note that, using Jensen's inequality, $f_t(\x^{\CBCE\la\cB\ra}_t) \le \la \bfell_t,\p_t \ra_{\bfcI_t}$.
  Then, in the proof of Theorem~\ref{thm:cblea-sleeping}
  \begin{align*}    
  &\sum_{J\in\cJ} \pi_{\cB_J} \cI_{t,\cB_J}\tilg_{t,\cB_J} w_{t,\cB_J}  \\
  &= \sum_{J\in\cJ: \pi_{\cB_J}\cI_{t,\cB_J} w_{t,\cB_J} > 0} \pi_{\cB_J} [w_{t,\cB_J}]_+ (f_t(\x^{\CBCE\la\cB\ra}_t) - \ell_{t,\cB_J}) \;+ \\
  &\qquad   \sum_{J\in\cJ: \pi_{\cB_J}\cI_{t,\cB_J} w_{t,\cB_J}\le 0} \pi_{\cB_J} \cI_{t,\cB_J} w_{t,\cB_J}[\la \bfell_t,\p_t \ra_{\bfcI_t} - \ell_{t,\cB_J}]_+ \\
  &\le \sum_{J\in\cJ: \pi_{\cB_J}\cI_{t,\cB_J} w_{t,\cB_J} > 0} \pi_{\cB_J} [w_{t,\cB_J}]_+ (\la \bfell_t,\p_t \ra_{\bfcI_t} - \ell_{t,\cB_J}) \;+ \\
  &\qquad   \sum_{J\in\cJ: \pi_{\cB_J}\cI_{t,\cB_J} w_{t,\cB_J}\le 0} \pi_{\cB_J} \cI_{t,\cB_J} w_{t,\cB_J}[f_t(\x^{\CBCE\la\cB\ra}_t) - \ell_{t,\cB_J}]_+ \;.
  \end{align*}
  Then, one can see that the proof of Theorem~\ref{thm:cblea-sleeping} goes through, so does Corollary~\ref{cor:cblea-sleeping}.
  
  Since $\KL(\e_{\cB_J} || \bar\bfpi) = \ln\fr{1}{\bar\pi_{\cB_J}} \le \ln\lt(\fr{\pi^2}{6}J_1^{2}(1+\lfl\log_2 J_1\rfl)\rt) \le 3 \ln(J_2) + \fr{1}{2}$, it follows that
  \begin{align*}
  \sum_{t\in J} f_t(\x^{\emph{\CBCE}\la\cB\ra}_t) - f_t(\x^{\cB_J}_t)
  &\stackrel{\text{(Cor.~\ref{cor:cblea-sleeping})}}{\le} \sqrt{2 S_{T,\cB_J} \cdot \lt( \text{{KL}}(\e_{\cB_J} || \bfpi) + \fr{1}{2} \ln(J_2) + 2 \rt) }  \\
  &\le \sqrt{2 |J| \lt( \fr{7}{2} \ln(J_2) + \fr{5}{2} \rt) }  \\
  &= \sqrt{|J| \lt( 7\ln(J_2) + 5\rt) }  \; .
  \end{align*}
\end{proof} 

\section{Proof of Theorem~\ref{thm:untitled}}
\begin{proof}

By Lemma~\ref{lem:geo-intv}, we know that $J$ can be decomposed into two sequences of intervals $\{J^{(-a)}, \ldots, J^{(0)}\}$ and $\{J^{(1)}, J^{(2)}, \ldots, J^{(b)}\}$.
Continuing from~\eqref{regret-decomposition},
\begin{align}
  R^{\CBCE\la\cB\ra}_I(\w)  
&= \underbrace{ \sum_{i=-a}^{b} \sum_{t\in J^{(i)}} \lt(f_t(\x^{\CBCE\la\cB\ra}_t) - f_t(\x^{\cB_{J^{(i)}}}_t) \rt) }_{S_1}  + 
   \underbrace{ \sum_{i=-a}^{b} \sum_{t\in J^{(i)}}  \lt(f_t(\x^{\cB_{J^{(i)}}}_t) - f_t(\w)\rt) }_{S_2} \;. \notag
\end{align}
Then,
\begin{align*}
S_1 = \sum_{i\in [(-a)..0]} \sum_{t \in J^{(i)}} \lt(f_t(\x^{\CBCE\la\cB\ra}_t) - f_t(\x^{\cB_{J^{(i)}}}_t) \rt)  + 
  \sum_{i\in [1..b]}\sum_{t \in J^{(i)}} \lt(f_t(\x^{\CBCE\la\cB\ra}_t) - f_t(\x^{\cB_{J^{(i)}}}_t) \rt) \;.
\end{align*}
The first summation is upper-bounded by, due to Lemma~\ref{lem:cbce} and Lemma~\ref{lem:geo-intv}, $\sum_{i\in [(-a)..0]} \sqrt{|J^{(i)}|(7\ln(I_2 + 5))} \le \sqrt{7\ln(I_2) + 5} \cdot \sum_{i=0}^\infty (2^{-i}|I|)^{1/2} \le \sqrt{7\ln(I_2) + 5} \cdot  (4 \sqrt{|I|}) $.
The second summation is bounded by the same quantity due to symmetry.
Thus, $S_1 \le 8\sqrt{|I| (7\ln( I_2) + 5)}$ .

In the same manner, one can show that $S_2 \le 2 \cdot \fr{2^\alpha}{2^\alpha - 1}|I|^\alpha \le \fr{4}{2^\alpha - 1}|I|^\alpha $, which concludes the proof.
\end{proof}

\section{The Data Streaming Intervals Can Replace the Geometric Covering Intervals}
\label{sec:ds}

We show that the data streaming intervals achieves the same goal as the geometric covering intervals (GC).
The data streaming intervals (DS) are 
\begin{align}\label{datastreaming}
\cJ = \{[t..(t+g\cdot2^{u(t)}-1)]: t =1,2,\ldots \} \;.
\end{align}
For any interval $J$, we denote by $J_1$ its starting time and by $J_2$ its ending time.
We say an interval $J'$ is a prefix of $J$ if $J'_1 = J=1$ and $J' \subseteq J$.

We show that DS also partitions an interval $I$ in Lemma~\ref{lem:datastreaming}.
\begin{lem}\label{lem:datastreaming}
  Consider $\cJ$ defined in~\eqref{datastreaming} with $g\ge1$.
  An interval $[I_1..I_2]\subseteq[T]$ can be partitioned to a sequence of intervals $\bar J^{(1)}, \bar J^{(2)}, \ldots, \bar J^{(n)}$ such that 
  \begin{enumerate}
    \item $\bar J^{(i)}$ is a prefix of some $J \in \cJ$.
    \item $|\bar J^{(i+1)}| / |\bar J^{(i)}| \ge 2$ for $i=1,\ldots,(n-1)$. 
  \end{enumerate}
\end{lem}
\begin{proof}
 For simplicity, we assume $g=1$; we later explain how the analysis can be extended to $g>1$.
 Let $I_1 = 2^u\cdot k$ where $2^u$ is the largest power of 2 that divides $I_1$.
 It follows that $k$ is an odd number.

 Let $J\in \cJ$ be the data streaming interval that starts from $I_1$.
 The length $|J|$ is $2^u$ by the definition, and $J_2$ is $I_1 + 2^u - 1$.
 Define $\bar J^{(1)} := J$.

 Then, consider the next interval $J'\in\cJ$ starting from time $I_1 + 2^u$.
 Note
 $$ J'_1 = I_1 + 2^u = 2^u \cdot k + 2^u = 2^{u+1} \cdot \fr{k+1}{2}  $$
 Note that $\fr{k+1}{2}$ is an integer since $k$ is odd.
 Therefore, $J'_1 = 2^{u'} \cdot k'$ where $u' > u$.
 It follows that the length of $J'$ is
 $$ |J'| = 2^{u'} \ge 2\cdot 2^u \;. $$
 Then, define $\bar J^{(2)} := J'$.

 We repeat this process until $I$ is completely covered by $\bar J^{(1)}, \ldots \bar J^{(n)}$ for some $n$.
 Finally, modify the last interval $\bar J^{(n)}$ to end at $I_2$ which is still a prefix of some $J\in \cJ$.
 This completes the proof for $g=1$.

 For the case of $g>1$, note that by setting $g>1$ we are only making the intervals longer.
 Observe that even if $g>1$, the sequence of intervals $\bar J^{(1)},\ldots,\bar J^{(n)}$ above are still prefixes of some intervals in $\cJ$.
\end{proof}

Note that, unlike the partition induced by GC in which interval lengths successively double then successively halve, the partition induced by DS just successively doubles its interval lengths except the last interval.
One can use DS to decompose SA-Regret of $\cM\la\cB\ra$; that is, in~\eqref{regret-decomposition}, replace $\sum_{i=-a}^b$ with $\sum_{i=1}^n$ and $J^{(i)}$ with $\bar J^{(i)}$.
Since the decomposition by DS has the same effect of ``doubling lengths', one can show that Theorem~\ref{thm:untitled} holds true with DS, too, with slightly smaller constant factors.

\section{A Subtle Difference between the Geometric Covering and Data Streaming Intervals}

There is a subtle difference between the geometric covering intervals (GC) and the data streaming intervals (DS).

As far as the black-box algorithm has an anytime regret bound, both GC and DS can be used to prove the overall regret bound as in Theorem~\ref{thm:untitled}.
In our experiments, the blackbox algorithm has anytime regret bound, so using DS does not break the theoretical guarantee.

However, there exist algorithms with fixed-budget regret bounds only.
That is, the algorithm needs to know the target time horizon $T^*$, and the regret bound exists after exactly $T^*$ time steps only.
When these algorithms are used as the black-box, there is no easy way to prove Theorem~\ref{thm:untitled} with DS intervals.
The good news, still, is that most online learning algorithms are equipped with anytime regret bounds, and one can often use a technique called `doubling-trick'~\cite[Section~2.3]{cesa-bianchi06prediction} to turn an algorithm with a fixed budget regret into the one with an anytime regret bound.

%
%

\end{document}